\documentclass[10pt,twocolumn]{article}

\usepackage[margin=1in]{geometry}
\usepackage{microtype}

\usepackage{amsmath,amssymb,mathtools,amsthm}

\usepackage{graphicx}
\usepackage{subcaption}
\usepackage{booktabs}
\usepackage{makecell}
\setcellgapes{2pt}\makegapedcells

\usepackage{algorithm}
\usepackage{algpseudocode}
\usepackage{enumitem}

\usepackage{tikz}
\usetikzlibrary{arrows.meta,positioning,fit,backgrounds}

\usepackage[numbers,sort&compress]{natbib}
\usepackage[hidelinks]{hyperref}
\usepackage[capitalize,noabbrev]{cleveref}

\newtheorem{theorem}{Theorem}[section]
\newtheorem{proposition}[theorem]{Proposition}
\newtheorem{lemma}[theorem]{Lemma}
\newtheorem{corollary}[theorem]{Corollary}
\theoremstyle{definition}
\newtheorem{definition}[theorem]{Definition}
\newtheorem{assumption}[theorem]{Assumption}
\theoremstyle{remark}

\title{Portfolio Reinforcement Learning with Scenario-Context Rollout}

\author{
Vanya Priscillia Bendatu\\
National University of Singapore\\
\texttt{vanyabendatu@u.nus.edu}
\and
Yao Lu\\
National University of Singapore\\
\texttt{luyao@comp.nus.edu.sg}
}
\date{}

\begin{document}
\maketitle
\pagestyle{plain} 
\begin{abstract}
Market regime shifts induce distribution shifts that can degrade the performance of portfolio rebalancing policies. We propose macro-conditioned scenario-context rollout (SCR) that generates plausible next-day multivariate return scenarios under stress events. However, doing so faces new challenges, as history will never tell what would have happened differently. As a result, incorporating scenario-based rewards from rollouts introduces a reward--transition mismatch in temporal-difference learning, destabilizing RL critic training. 

We analyze this inconsistency and show it leads to a mixed evaluation target. Guided by this analysis, we construct a counterfactual next state using the rollout-implied continuations and augment the critic agent's bootstrap target. Doing so stabilizes the learning and provides a viable bias-variance tradeoff.

In out-of-sample evaluations across 31 distinct universes of U.S. equity and ETF portfolios, our method improves Sharpe ratio by up to 76\% and reduces maximum drawdown by up to 53\% compared with classic and RL-based portfolio rebalancing baselines. 
\end{abstract}

\section{Introduction}
\label{sec:intro}

The financial market is non‑stationary. Over time, the expected returns, volatilities, and cross‑asset dependency change as the macroeconomic structure and the market dynamics evolve. In portfolio management and rebalancing applications, regime changes may cause degraded returns due to the dependencies during downturns \citep{longin2001,CampbellKoedijkKofman2002,AngChen2002}. A real-world portfolio policy must be resilient to distribution shifts and synchronized sell‑offs \citep{levine2020offline,fujimoto2019offpolicy,kumar2020cql}.

Risk management traditionally addresses robustness through scenario analysis and stress testing. Agent-based macro modeling and simulation~\citep{poledna2023economic,westerhoff2018agent} push this idea further such that policies can be evaluated under counterfactual scenarios while preserving economically meaningful structures. On the other hand, reinforcement learning (RL) enables portfolio rebalancing and treats the asset allocation problem as a Markov decision process (MDP), aiming to maximize the return while following trading limits and constraints \citep{moody2001learning,jiang2017portfolio,fischer2018deep,sutton2018reinforcement}. Recent policy‑gradient methods, such as proximal policy optimization (PPO) \citep{Schulman2017PPO}, stabilize training with clipped objectives and bootstrapped value targets.

In this paper, we draw inspiration from agent-based macro modeling and explore this idea in training RL agents for portfolio management. To improve next-state estimation, we propose Scenario-Context Rollout (SCR), a leak-safe feedback mechanism to produce a distribution of next-day joint returns under potential economic shocks. However, doing so poses a key technical challenge: while outcomes can be evaluated under plausible simulated scenarios, the agent observes only the realized transitions from the market history and cannot interact with the market to test what‑if actions \citep{levine2020offline}. As a result, alternative trajectories are only hypothetical. This problem worsens in distribution shifts; the policy may generate state and action pairs that are not supported by economic structures \citep{kumar2019bear,kumar2020cql}.

To deepen our comprehension of this problem, we in theory formalize {scenario-conditioned learning} with historical market data and show that coupling scenario-based rewards with tape-based (realized history) continuations induces {hybrid Bellman operator} \citep{Puterman1994,YuMahmoodSutton2018} under outcome-dependent state memory and causes a reward–transition mismatch in temporal-difference (TD) learning. Accordingly, we use the Wasserstein-form operator bounds \citep{Villani2008optimal,FernsPanangadenPrecup2004} to prove that this hybrid operator generally has a distinct fixed point from the {scenario-consistent} objective, creating an inherent fixed-point bias. This divergence is not only a nuisance but also exposes a bias--variance tradeoff in critic training.  

Based on this analysis, we propose to augment the critic's bootstrap target with a counterfactual continuation that is constructed from the SCR returns. The resulting target interpolates between a tape and a scenario bootstrap; we show that doing so improves training stability and robustness while adhering to economic structure and historical observations.

We evaluate our solution on day-to-day portfolio rebalancing across 31 distinct universes of U.S. equity and ETF portfolios. We compare against classic portfolio rebalancing algorithms, tape-only PPO, risk-penalized, and rollout-based baselines. Results show that our solution improves out-of-sample Sharpe ratio by up to 76\% and reduces maximum drawdown by over 53\%, while providing a more stable RL training.

\vspace{0.05in}\noindent\textbf{Contributions} of this paper can be summarized as follows:
\begin{itemize}[leftmargin=*,nosep]
\item We introduce SCR for RL-based portfolio management. SCR is a macroeconomics-guided feedback mechanism that produces a distribution of returns under economic shock scenarios.
\item We analyze in theory that scenario rewards together with tape continuations induce a hybrid Bellman operator and yield a fixed-point bias characterized by Wasserstein-form bounds. This implies a bias-variance tradeoff in critic training. We design our solution based on this analysis.
\item Our solution shows empirical gains using historical market data over state-of-the-art portfolio rebalancing algorithms and RL models.
\end{itemize}

\section{Method}
\label{sec:methods}
Rebalancing close-to-close trading portfolios is a sequential decision problem. At the close of every trading day $t$, we pick a vector of weights  $\mathbf w_t \in \mathbb{R}^N$ across $N$ assets. 
To calculate the performance of the portfolio, we use the asset return vector $\mathbf r_{t+1} \in \mathbb{R}^N$. Each individual element $r_{t+1, i}$ of this vector is defined as 
\begin{equation}
    r_{t+1,i} := \frac{p_{t+1,i} -p_{t,i}}{p_{t,i}}, \quad i = 1, \dots, N,
\end{equation}
where $p_{t,i}$ represents the closing price of asset $i$ on day $t$. The final realized net portfolio return is then given by
\begin{equation}
    R_{p,t+1} = \mathbf w^\top_t \mathbf r_{t+1} - \text{cost}_t,
\end{equation}
\noindent where $\text{cost}_t$ is a term that increases with portfolio turnover \citep{olivaresnadal2018,vassallo2022}.

\begin{figure*}[t]
  \centering
  \includegraphics[angle=-90, width=0.85\textwidth]{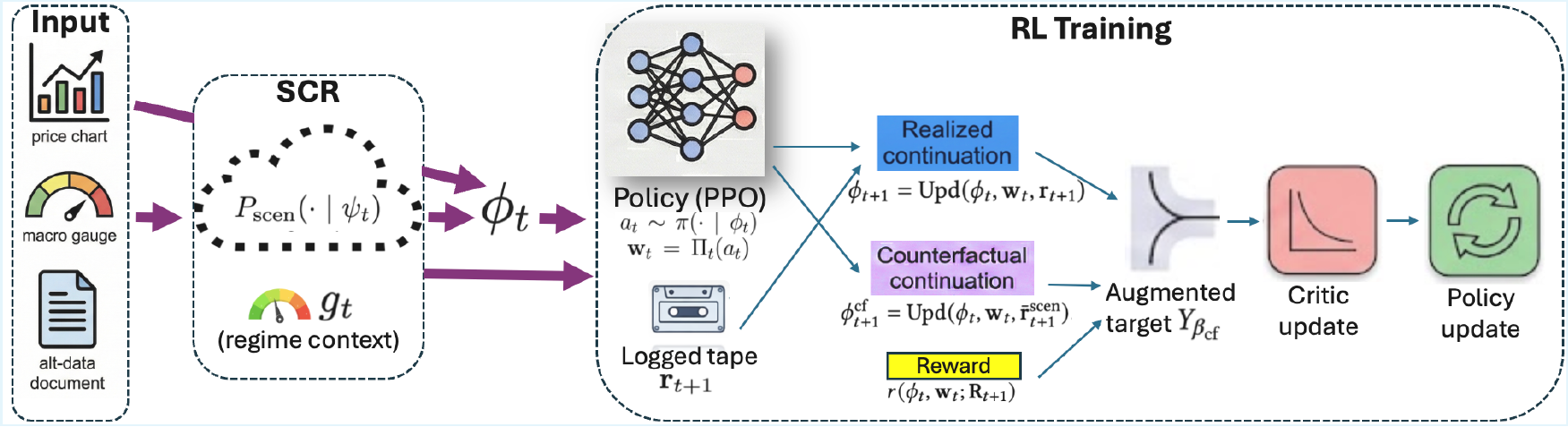}
  \caption{\textbf{Overview of our RL paradigm.} \textsc{SCR} produces a conditional scenario distribution over next-day joint return vectors. The critic is trained using a bootstrap target that augments the realized continuation bootstrap with a counterfactual continuation. }
\label{fig:arch_main}
\end{figure*}

Figure \ref{fig:arch_main} summarizes our scenario-conditioned RL pipeline.
We follow the standard PPO actor-critic training paradigm \citep{Schulman2017PPO} and utilize a learned value function together with a critic bootstrap target. Our objective is to enhance robustness in the face of regime shifts and changes in distribution.

Specifically, we gather regime/exposure features which serve as a retrieval descriptor denoted by $\psi_t$. Crucially, $\psi_t$ excludes portfolio-history components (e.g., previous weights or internal memory states), ensuring that scenario retrieval relies solely on market conditions instead of the agent’s past trajectory. Therefore, Scenario-Context Rollout (SCR) estimates a conditional distribution $P_{\mathrm{scen}}(\cdot\mid\psi_t)$ over next-day joint return vectors induced by {macro scenarios}.
Let $\mathbf R_{t+1}\sim P_{\mathrm{scen}}(\cdot\mid\psi_t)$ denote a generic macro scenario draw.
To approximate expectations, we sample independent and identically distributed (i.i.d.) macro scenarios
\[
\mathbf R^{(s)}_{t+1}\stackrel{\text{i.i.d.}}{\sim} P_{\mathrm{scen}}(\cdot\mid\psi_t),\qquad s=1,\dots,S.
\]
We define $\bar{\mathbf r}^{\mathrm{scen}}_{t+1}$ as a compact representation used in the counterfactual continuation (Sec.~\ref{subsec:methods_ccm}).
This joint-return, conditioned by the prevailing regime, captures the dependencies and tail risks that become prominent in times of stress.

At each decision time $t$, the policy observes an input state $\phi_t=(z_t,h_t)$ containing only information available by time $t$.
In this context, $z_t$ encompasses market features and portfolio history
and $h_t$ provides longer-horizon context including SCR outputs.
The policy model generates an action $a_t$. To ensure that the final portfolio is valid, we apply a set of operational constraints $\mathcal{W}_t$, covering box, leverage, and turnover constraints: \begin{equation}
\mathbf w_t = \Pi_t(a_t)\in\mathcal W_t,
\label{eq:feas_exec}
\end{equation}
 where $\Pi_t$ denotes the projection onto $\mathcal W_t$. All constraint budgets are treated as held identical across all strategies.

Let $r(\phi,\mathbf w;\mathbf x)$ denote the per-step reward evaluated at state $\phi$ using weights $\mathbf w$ and based on an input return vector $\mathbf x$. We train PPO to maximize a discounted objective where rewards are evaluated under \textsc{SCR} macro scenarios, which involve a conditional distribution of plausible next-day joint returns:
\begin{equation}
\label{eq:jscen_clean}
J_{\mathrm{scen}}(\pi)
:=\mathbb E_{\pi,P_{\mathrm{scen}}}\!\left[
\sum_{t=0}^{T-1}\delta^t\, r(\phi_t,\mathbf w_t;\mathbf R_{t+1})
\right].
\end{equation}
The concrete form of reward $r(\phi_t,\mathbf w_t;\mathbf R_{t+1})$ is defined in Sec.~\ref{subsec:methods_ccm}.
We use a discount factor $\delta\in(0,1)$ to ensure stability in temporal-difference (TD) bootstrapping. That is, in discounted MDPs, the Bellman operator acts as a contraction with a modulus of $\delta$. This enhances the conditioning of value targets and mitigates sensitivity to noise arising from long-horizon estimations  (e.g.,~\cite{sutton2018reinforcement, bertsekas1996ndp}).

It is important to recognize that a logged market tape (i.e., a fixed, non-rewinding historical record) provides only a single realized next-day return vector $\mathbf r_{t+1}$ at each time $t$,
and thus only a single realized continuation.
By utilizing executed weights $\mathbf w_t$, we update the realized next state through function $\phi_{t+1}=\mathrm{Upd}(\phi_t,\mathbf w_t,\mathbf r_{t+1})$. Conversely, Eq.~\eqref{eq:jscen_clean} specifies \emph{how actions are scored} in macro scenario estimation, but the logged tape determines the realized continuation.
This creates a mismatch between (i) scenario-based reward scoring (as in $J_{\mathrm{scen}}(\pi)$) and (ii) the realized next-state transition $\phi_{t+1}$.
We fix the gap here using the macro scenario conditional mean
$\bar{\mathbf r}^{\mathrm{scen}}_{t+1}$ by constructing a counterfactual continuation: $\phi^{\mathrm{cf}}_{t+1}=\mathrm{Upd}(\phi_t,\mathbf w_t,\bar{\mathbf r}^{\mathrm{scen}}_{t+1})$,
and then mixing the critic bootstraps evaluated at $\phi_{t+1}$ and $\phi^{\mathrm{cf}}_{t+1}$ in the augmented TD target $Y_{\beta_{\mathrm{cf}}}$
(see Sec.~\ref{subsec:methods_ccm}). This results in a bootstrap signal that is aligned with the scenario and exhibits low variance, while remaining grounded in the tape. Ultimately, PPO updates the critic based on $Y_{\beta_{\mathrm{cf}}}$ and updates the policy accordingly.

\subsection{Scenario-Context Rollout (\textsc{SCR})} 
\label{subsec:methods_SCR}
\textbf{Regime representation.}
Each trading day $t$ can be represented by a low-dimensional \emph{regime embedding} $U_t\in\mathbb R^k$ to summarize the market. $U_t$ is learned to capture patterns in two sources of information: (i) market--macro behavior (cross-sectional returns and macro shocks) and (ii) policy-text signals. Intuitively, nearby points in the shock space correspond to days with similar ``macro-financial regimes.'' A \emph{shock day} is hence a day whose regime embedding $U_t$ contains an anomaly; this can be learned during training as a binary indicator $\mathrm{shock}_t\in\{0,1\}$.

As shocks often happen sequentially, spanning days to months. We cluster subsequent shock days into a \emph{shock channel}, assigning each shock-day embedding to an existing cluster or creates a new one when it is farther than a distance threshold $\lambda^2$ from all existing \emph{channel centroids} using the running mean embedding of each channel.
We keep a \emph{ShockLedger}, a compact catalog of shock channels that encapsulates the channel’s macro signature and representative movers.
To obtain a \emph{daily} regime characteristic for conditioning, we aggregate recent channel hits into a binary activation pattern, considering the activated channels as active near time $t$.
This yields a channel-activation vector $\boldsymbol{\chi}_t\in\{0,1\}^C$, where $\chi_{t,c}=1$ if and only if channel $c$ is assigned to at least one shock day within a brief trailing lookback window concluding at time $t$.
The lookback aggregation reflects the enduring effects of macro shocks over several days. This reduces daily fluctuations in the output of the detector, resulting in a more stable conditioning signal.

\vspace{0.05in}\noindent\textbf{Scenario retrieval and rollout.}
Macro scenario retrieval is facilitated through channel activations using $\psi_t$ with the preprocessed regime and exposure data that is accessible at time $t$ (e.g., macro shocks, channel activations $\boldsymbol{\chi}_t$, and exposure summaries). To uphold leak-safety, $\psi_t$ keeps only necessary features and omits those related to the portfolio history. 
Hence, SCR maintains a library of return vectors indexed by a series of historical decision times $\mathcal I$:
\begin{equation}
\mathcal D
:= \{(\widetilde{\mathbf r}_{u+1}, \psi_u)\}_{u\in\mathcal I}, 
\label{eq:SCR_library}
\end{equation}
where $u<t$ denotes a past decision time.
The vector $\widetilde{\mathbf r}_{u+1}$ is a one-step-ahead \emph{macro scenario return} that is derived from the macro context through a rolling macro-to-return map, using only information available at time $u$.

At time $t$, \textsc{SCR} restricts to past indices $\mathcal I_t:=\{u\in\mathcal I:\ u<t\}$ and retrieves nearest neighbors
by comparing the current descriptor $\psi_t$ and past descriptors $\{\psi_u\}_{u\in\mathcal I_t}$ in $\mathcal D$. Let $\mathrm{sim}(\cdot,\cdot)$ be a similarity score. We define the  $k$ nearest-neighbor indices as
\begin{equation}
\mathcal N_k(t) := \operatorname{TopK}_{u\in\mathcal I_t}\ \mathrm{sim}(\psi_t,\psi_u).
\label{eq:SCR_topk}
\end{equation}
Therefore, SCR produces the conditional distribution based on the retrieval:
\begin{equation}
P_{\mathrm{scen}}(\cdot\mid\psi_t)
:= \frac{1}{|\mathcal N_k(t)|}\sum_{u\in\mathcal N_k(t)} \delta_{\widetilde{\mathbf r}_{u+1}}(\cdot),
\label{eq:SCR_pscen}
\end{equation}
where $\delta_x$ denotes a point mass at $x$.
During RL training, we draw $S$ i.i.d.\ macro scenarios $\{\mathbf R^{(s)}_{t+1}\}_{s=1}^S$ from $P_{\mathrm{scen}}(\cdot\mid\psi_t)$.

\vspace{0.05in}\noindent\textbf{Regime context  $g_t$.}
Beyond scenario sampling, \textsc{SCR} outputs a regime context $g_t$ that gates risk-taking at decision time $t$. It acts as a risk-budget multiplier: $g_t\approx 1$ corresponds to normal conditions, while smaller $g_t$ indicates stress and
shrinks scenario-evaluated portfolio payoffs, discouraging aggressive exposures. We compute a nonnegative severity score $v_t$ via a short-horizon deterministic \emph{macro stress rollout} driven by the active channels $\boldsymbol{\chi}_t$, tracking the peak Mahalanobis \citep{mahalanobis2018reprint} stress along the rollout. We then normalize  $v_t$ against a trailing quantile baseline, as it may drift over time.

\begin{definition}[Quantile-normalized regime context $g_t$] 
\label{def:regimecontext}
Let $v_t\ge 0$ be the severity score at decision time $t$, and let $\mathcal W_t:=\{t-L_g,\dots,t-1\}$ be a trailing window of past decisions of length $L_g$ and a quantile $q_g\in(0,1)$. 
Define the reference scale $q_t := \mathrm{Quantile}_{q_g}\big(\{v_u:\ u\in\mathcal W_t\}\big)$,
and the regime context
\begin{equation}
g_t
:= \mathrm{clip}\!\left(1-\alpha_g\frac{v_t}{q_t+\varepsilon},\underline g,1\right)\in[\underline g,1],
\label{eq:regimecontext}
\end{equation}
where $\alpha_g>0$ controls sensitivity, $\underline g\in(0,1)$ sets a minimum value. We apply $g_t$ as an \emph{exposure gate} by scaling scenario portfolio returns
\[
\tilde u_{t+1}^{(s)} \;=\; g_t\,\langle \mathbf w_t,\mathbf R_{t+1}^{(s)}\rangle,
\]
so larger estimated severity (larger $v_t$) yields smaller $g_t$, attenuating the reward signal during periods of stress. 
In other words, it reduces risk-taking and provides smooth, effective risk management.
\end{definition}

\subsection{Counterfactual Continuation for Critic Target Augmentation}
\label{subsec:methods_ccm}

We use a deterministic function $\mathrm{Upd}:\ (\phi_t,\mathbf w_t,\mathbf x_{t+1})\mapsto \phi_{t+1}$, which advances the current state when fed with an \emph{input} return vector $\mathbf x_{t+1}\in\mathbb R^N$.\footnote{$\mathbf x_{t+1}$ is a placeholder for the return vector; on the tape we observe $\mathbf x_{t+1}=\mathbf r_{t+1}$.}
At time $t$, the agent selects an action $a_t$ and executes feasible portfolio weights
$\mathbf w_t=\Pi_t(a_t)$.
The tape provides the realized return vector $\mathbf r_{t+1}$ and the realized continuation
\begin{equation}
\phi_{t+1}
=\mathrm{Upd}\!\left(\phi_t,\mathbf w_t,\mathbf r_{t+1}\right),
\qquad \mathbf r_{t+1}\ \text{from tape}.
\label{eq:sc2b_real_state}
\end{equation}

This realized continuation mismatches with macro scenario estimates. We address this problem by constructing a \emph{counterfactual continuation} and augmenting the critic’s one-step bootstrap target.  
Specifically, given executed weights $\mathbf w_t$ and macro scenario samples
$\{\mathbf R_{t+1}^{(s)}\}_{s=1}^S\sim P_{\mathrm{scen}}(\cdot\mid\psi_t)$, we use the gated outcomes $\tilde u_{t+1}^{(s)}$ defined in Sec. ~\ref{subsec:methods_SCR} and their sample mean
\[
\bar{\tilde u}_{t+1}\;:=\;\frac{1}{S}\sum_{s=1}^S \tilde u_{t+1}^{(s)}.
\]
Integrating these macro scenario outcomes yields a risk-aware reward:
\begin{equation}
\begin{aligned}
r_t
:=\; & \bar{\tilde u}_{t+1}
-\lambda_{\rho}\Big(\mathrm{Risk}_{\eta}\!\left(\{\tilde u_{t+1}^{(s)}\}_{s=1}^S\right)+\eta\,\varepsilon\Big) \\
& - \mathrm{Reg}(\mathbf w_t,\mathbf w_{t-1}),
\end{aligned}
\label{eq:SCR_reward}
\end{equation}

where $\lambda_\rho\ge 0$ regulates the intensity of the tail-risk penalty, $\mathrm{Risk}_{\eta}$ represents a functional that is sensitive to tail conditions with a temperature parameter of $\eta>0$, and $\mathrm{Reg}$ accounts for trading frictions as well as structural penalties. 

The conditional mean return is hence approximated through Monte Carlo \citep{glasserman2004montecarlo}:
\begin{equation}
\bar{\mathbf r}^{\mathrm{scen}}_{t+1}
:= \mathbb E_{P_{\mathrm{scen}}(\cdot\mid\psi_t)}[\mathbf R_{t+1}]
\approx \frac{1}{S}\sum_{s=1}^S \mathbf R_{t+1}^{(s)}.
\label{eq:SCR_rbar}
\end{equation}
Using this approximation, we establish the counterfactual continuation by applying the same update function to $\bar{\mathbf r}^{\mathrm{scen}}_{t+1}$:
\begin{equation}
\phi^{\mathrm{cf}}_{t+1}
=\mathrm{Upd}\!\left(\phi_t,\mathbf w_t,\bar{\mathbf r}^{\mathrm{scen}}_{t+1}\right).
\label{eq:sc2b_cf_state}
\end{equation}
Notably, this counterfactual continuation is secured against data leakage as it relies solely on quantities available by time $t$. It is also aligned with macro scenario since $\bar{\mathbf r}^{\mathrm{scen}}_{t+1}$ is derived from the same macro scenario distribution $P_{\mathrm{scen}}(\cdot\mid\psi_t)$ utilized for the reward. 

\vspace{0.05in}\noindent\textbf{Augmenting Critic Bootstrap Target.}
Let $V_\omega(\phi)$ denote the critic, i.e., a parametric approximation to the state-value function under policy. Let $\delta\in(0,1)$ be the RL discount factor.
Using the same scenario-based reward $r_t$ from Eq.~\eqref{eq:SCR_reward}, let the one-step targets using (1) the tape-based realized continuation and (2) the counterfactual continuation respectively be
\[
Y_r := r_t + \delta\,V_\omega(\phi_{t+1}),
\qquad
Y_c := r_t + \delta\,V_\omega(\phi^{\mathrm{cf}}_{t+1}).
\]
Therefore we augment the critic target using
\begin{equation}
Y_{\beta_{\mathrm{cf}}}
=(1-\beta_{cf})\,Y_r+\beta_{cf}\,Y_c,
\qquad \beta_{cf}\in[0,1],
\label{eq:sc2b_mix}
\end{equation}
where $\beta_{cf}$ interpolates between tape-faithful bootstrapping
and scenario-aligned bootstrapping.
The critic is trained by regressing $V_\omega(\phi_t)$ onto $Y_{\beta_{\mathrm{cf}}}$, and the policy model undergoes updates through PPO-Clip with GAE advantages (Alg.~\ref{alg:sc2b_main}).

\begin{algorithm}[t]
\caption{Actor--critic training using logged tape with SCR macro scenarios and augmented targets}
\label{alg:sc2b_main}
\begin{algorithmic}[1]
\Require Logged tape $\{\phi_t,\mathbf r_{t+1}\}_{t=0}^{T-1}$; SCR module returning $(P_{\mathrm{scen}}(\cdot\mid\psi_t), g_t)$;
feasibility map $\Pi_t$; update $\mathrm{Upd}$; discount $\delta$; mixing $\beta_{cf}$; scenario sample size $S$
\For{each iteration of PPO update}
  \For{$t=0,\dots,T-1$ (traverse through logged dates)}
    \State Construct $\phi_t$; sample $a_t\sim\pi_\theta(\cdot\mid\phi_t)$;  $\mathbf w_t=\Pi_t(a_t)$
    \State Obtain $(P_{\mathrm{scen}}(\cdot\mid\psi_t), g_t)$ from the SCR module
    \State Sample $\{\mathbf R_{t+1}^{(s)}\}_{s=1}^S\sim P_{\mathrm{scen}}(\cdot\mid\psi_t)$
    \State Compute scenario reward $r_t$ using $\{\mathbf R_{t+1}^{(s)}\}$, $g_t$ (Eq.~\eqref{eq:SCR_reward})
    \State Compute scenario mean $\bar{\mathbf r}^{\mathrm{scen}}_{t+1}=\frac{1}{S}\sum_{s=1}^S \mathbf R_{t+1}^{(s)}$ (Eq.~\eqref{eq:SCR_rbar})
    \State Read realized $\mathbf r_{t+1}$ from tape; $\phi_{t+1}=\mathrm{Upd}(\phi_t,\mathbf w_t,\mathbf r_{t+1})$
    \State Set $\phi^{\mathrm{cf}}_{t+1}=\mathrm{Upd}(\phi_t,\mathbf w_t,\bar{\mathbf r}^{\mathrm{scen}}_{t+1})$ (Eq.~\eqref{eq:sc2b_cf_state})
    \State Form augmented target $Y_{\beta_{\mathrm{cf}}}$ using Eq.~\eqref{eq:sc2b_mix}
  \EndFor
  \State Update critic by regressing $V_\omega(\phi_t)$ onto $Y_{\beta_{\mathrm{cf}}}$
  \State Compute Generalized Advantage Estimation $\hat A_t$
  \State Update actor using PPO-Clip with $\hat A_t$
\EndFor
\end{algorithmic}
\end{algorithm}

\section{Continuation Mismatch and Fixes}
\label{sec:theory}

This section discusses the \emph{continuation mismatch} at the Bellman operator level \citep{Puterman1994,sutton2018reinforcement,bertsekas1996ndp,munos2005avi}
and explains why \emph{augmenting targets on the critic side} is both adequate and well-founded.

\vspace{0.05in}\noindent\textbf{Hybrid vs.\ scenario-consistent operators.}
Let a policy $\pi$ be defined along with any bounded value function $V$.
According to Sec. \ref{sec:methods}, $\psi$ be the leak-safe descriptor used by \textsc{SCR} and
$P_{\mathrm{scen}}(\cdot\mid\psi)$ for the induced scenario return law.
Let $P_{\mathrm{real}}(\cdot\mid\phi)$ denote the (conceptual) real-world conditional return law for the tape outcome.

We consider two Bellman-style operators that differ only in their bootstrapping techniques:
the \emph{hybrid} operator evaluates \emph{scenario reward} but bootstraps on the \emph{realized} continuation, while the
\emph{scenario-consistent} operator uses scenario outcomes for both reward and continuation:
\begin{equation}
(T_{\mathrm{hyb}}^\pi V)(\phi)
:= \mathbb E\!\left[
r(\phi,\mathbf w;\mathbf R)+\delta\,V\!\left(\mathrm{Upd}(\phi,\mathbf w,\mathbf r)\right)
\ \middle|\ \phi
\right],
\label{eq:Thyb}
\end{equation}
where $a\sim\pi(\cdot\mid\phi)$, $\mathbf w=\Pi(a)$, $\mathbf R\sim P_{\mathrm{scen}}(\cdot\mid\psi)$, and
$\mathbf r\sim P_{\mathrm{real}}(\cdot\mid\phi)$.
In contrast,
\begin{equation}
(T_{\mathrm{scen}}^\pi V)(\phi)
:= \mathbb E\!\left[
r(\phi,\mathbf w;\mathbf R)+\delta\,V\!\left(\mathrm{Upd}(\phi,\mathbf w,\mathbf R)\right)
\ \middle|\ \phi
\right],
\label{eq:Tscen_mismatch}
\end{equation}
with the same action sampling but using the scenario outcome within its continuation.
Thus $T_{\mathrm{hyb}}^\pi$ corresponds to \emph{scenario reward + realized continuation}, while $T_{\mathrm{scen}}^\pi$ corresponds to
\emph{scenario reward + scenario continuation}.
When $\mathrm{Upd}$ depends on the newly observed next-day return via the memory component $h$ and
$P_{\mathrm{scen}}(\cdot\mid\psi)\neq P_{\mathrm{real}}(\cdot\mid\phi)$, these operators typically converge to different fixed points.
This is because they take evaluation of the continuation term $V(\mathrm{Upd}(\phi,\mathbf w,\cdot))$ under different outcome distributions.

\begin{lemma}
\label{lem:reward_cancels}
Fix $\phi$ and take any bounded $V$. We define the function
\begin{equation}
f_{\phi,a}(\mathbf x)\coloneqq V\!\big(\mathrm{Upd}(\phi,\Pi(a),\mathbf x)\big).
\end{equation}
Then for any policy $\pi$,
\begin{equation}
\label{eq:reward_cancels_compact}
\begin{aligned}
(T_{\mathrm{hyb}}^\pi V - T_{\mathrm{scen}}^\pi V)(\phi)
&=\delta\,\mathbb E_{a\sim\pi(\cdot\mid\phi)}\!\big[\Delta(\phi,\psi,a)\big],\\
\Delta(\phi,\psi,a)
&:= \mathbb E_{\mathbf r\sim P_{\mathrm{real}}(\cdot\mid\phi)}
\big[f_{\phi,a}(\mathbf r)\big] \\
&\quad - \mathbb E_{\mathbf R\sim P_{\mathrm{scen}}(\cdot\mid\psi)}
\big[f_{\phi,a}(\mathbf R)\big].
\end{aligned}
\end{equation}

\end{lemma}

\begin{proposition}
\label{prop:op_gap}
Assume that for each fixed $(\phi,\mathbf w)$ the map $\mathbf x\mapsto \mathrm{Upd}(\phi,\mathbf w,\mathbf x)$ is $L_h$-Lipschitz in $\mathbf x$
through the outcome-updated memory component $h$, i.e.,
$\|h(\mathrm{Upd}(\phi,\mathbf w,\mathbf x)) - h(\mathrm{Upd}(\phi,\mathbf w,\mathbf y))\|
\le L_h\|\mathbf x-\mathbf y\|_2$. Also, assume $V$ is $L_V$-Lipschitz in that memory (holding exogenous context fixed).
Let $o$ denote the information available at the decision time, $\phi=\Phi(o)$ and $\psi=\Psi(o)$ be summaries constructed from the same $o$.
Let $\delta\in(0,1)$ denote the discount factor used in the operators.

Define the mismatch by
\begin{equation}
\Delta_W
:= \sup_{o}\ 
W_1\!\Big(P_{\mathrm{real}}(\cdot \mid \Phi(o)),\;
          P_{\mathrm{scen}}(\cdot \mid \Psi(o))\Big),
\label{eq:DeltaW}
\end{equation}
where $W_1$ is the 1-Wasserstein distance with ground metric $d(\mathbf x,\mathbf y)=\|\mathbf x-\mathbf y\|_2$ \citep{peyre2019computational}.
Then for any policy $\pi$ and any such $V$, the resulting one-step operator deviation is bounded by
\begin{equation}
\big\|T_{\mathrm{hyb}}^\pi V - T_{\mathrm{scen}}^\pi V\big\|_\infty
\ \le\
\delta\,L_V L_h\,\Delta_W.
\label{eq:op_gap}
\end{equation}
\end{proposition}

\begin{corollary}

\label{cor:fixed_point_bias}
Let $V_{\mathrm{scen}}^\pi$ be the fixed point of $T_{\mathrm{scen}}^\pi$ and $V_{\mathrm{hyb}}^\pi$ be the fixed point of $T_{\mathrm{hyb}}^\pi$, assuming both are $\delta$-contractions on $(\mathcal V,\|\cdot\|_\infty)$.
Assume $V_{\mathrm{scen}}^\pi$ is $L_V$-Lipschitz in the outcome-updated memory component $h$, and define
\[
L_V \;:=\; \mathrm{Lip}_h\!\left(V_{\mathrm{scen}}^\pi\right).
\]
Under the conditions of Proposition~\ref{prop:op_gap},
\begin{align}
\|V_{\mathrm{hyb}}^\pi - V_{\mathrm{scen}}^\pi\|_\infty
&\le
\frac{\delta}{1-\delta}\,L_V L_h\,\Delta_W,
\label{eq:fixed_point_bias}
\end{align}
with $\Delta_W$ from \eqref{eq:DeltaW}.
\end{corollary}

A conflict inherent to logged tape learning is shown in Corollary~\ref{cor:fixed_point_bias}. Applying macro scenarios directly to the reward creates immediate friction with the critic. The fundamental problem is that we are scoring actions under SCR macro scenarios, even though we still depend on the tape for bootstrapping. Consequently, the TD update gradually deviates. It follows the \emph{hybrid} operator and converges (in expectation) to
$V^\pi_{\mathrm{hyb}}$, not to the scenario-consistent value, $V^\pi_{\mathrm{scen}}$, required by the scoring objective $J_{\mathrm{scen}}$
(Eq.~\eqref{eq:jscen_clean}).
Lemma~\ref{lem:reward_cancels} isolates the core issue, showing that the mismatch stems entirely from the continuation term. This is a crucial detail since it means we can adjust the critic bootstrap without actually interfering with the tape traversal. We resolve the issue by generating a counterfactual continuation state and combining its value with the realized tape bootstrap. This creates an augmented target that reduces the mismatch. As a result, following the logic in Sec.~\ref{subsec:methods_ccm}, we can rectify the critic while ensuring our rollouts strictly faithful to the tape.

\begin{assumption}
\label{ass:sc2b_reg}
Let $\phi=(z,h)$, where $h$ is the outcome-updated memory component and $z$ denotes exogenous context.
For any fixed $(\phi,\mathbf w)$, assume the update map $\mathbf x\mapsto \mathrm{Upd}(\phi,\mathbf w,\mathbf x)$ is $L_h$-Lipschitz in $\mathbf x$
through $h$. Likewise, the scenario-consistent value $V^\pi_{\mathrm{scen}}$ is $L_V$-Lipschitz in $h$ (holding $z$ fixed).
Finally, the mismatch and proxy second moments are assumed to be finite:
\[
\begin{aligned}
W_2^2\!\Big(P_{\mathrm{real}}(\cdot\mid\phi),\,P_{\mathrm{scen}}(\cdot\mid\psi)\Big) &< \infty,\\
\mathbb E\!\left[\|\mathbf R-\mu_{\psi}\|_2^2 \,\middle|\, \psi\right] &< \infty,
\end{aligned}
\]
where 
$\mathbf R\sim P_{\mathrm{scen}}(\cdot\mid\psi)$ and $\mu_{\psi}:=\mathbb E[\mathbf R\mid\psi]$.

\end{assumption}

We bound the one-step mean squared error between the augmented target and the scenario-consistent target $Y^\star$ in accordance with Assumption~\ref{ass:sc2b_reg}.

\begin{theorem}[One-step mixing bound (Wasserstein form)]
\label{thm:sc2b_mse}
Consider a decision time with state $\phi$, descriptor $\psi$, and a raw action $a$ with executed control $\mathbf w=\Pi(a)$.
Let the scenario-consistent one-step target be
\[
\begin{aligned}
Y^\star
&:= r(\phi,\mathbf w;\mathbf R)
  + \delta\,V^\pi_{\mathrm{scen}}\!\Big(\mathrm{Upd}(\phi,\mathbf w,\mathbf R)\Big),\\
&\qquad \mathbf R \sim P_{\mathrm{scen}}(\cdot\mid\psi).
\end{aligned}
\]
Define $Y_{\beta_{\mathrm{cf}}}$ as in \eqref{eq:sc2b_mix} using $V^\pi_{\mathrm{scen}}$ in place of $V_\omega$ and
$\phi^{\mathrm{cf}}=\mathrm{Upd}(\phi,\mathbf w,\mu_\psi)$.
Under Assumption~\ref{ass:sc2b_reg}, define
\[
\begin{aligned}
\Delta_2(\phi,\psi)
&:= W_2^2\!\Big(P_{\mathrm{real}}(\cdot\mid\phi),\,P_{\mathrm{scen}}(\cdot\mid\psi)\Big),\\
\sigma^2_{\mathrm{scen}}(\psi)
&:= \mathbb E\!\left[\|\mathbf R-\mu_\psi\|_2^2 \,\middle|\, \psi\right].
\end{aligned}
\]
Then
\begin{multline}
\label{eq:sc2b_mse_bound}
\mathbb{E}\!\left[(Y_{\beta_{\mathrm{cf}}}-Y^\star)^2 \,\middle|\, \phi,\psi,a\right] \\
\shoveleft{\le 2\,\delta^2(L_VL_h)^2
\Bigl(
(1-\beta_{cf})^2\Delta_2(\phi,\psi)
+ \beta_{cf}^{\,2}\sigma^2_{\mathrm{scen}}(\psi)
\Bigr).}
\end{multline}

\end{theorem}

\begin{corollary}
\label{cor:sc2b_beta_star}
Let $A(\phi,\psi):=\Delta_2(\phi,\psi)$ denote the mismatch and $B(\psi):=\sigma^2_{\mathrm{scen}}(\psi)$ the proxy variance.
Minimizing the right hand side of \eqref{eq:sc2b_mse_bound} yields the optimal mixing weight:
\begin{equation}
\beta^\star_{\mathrm{cf}}(\phi,\psi)=\frac{A(\phi,\psi)}{A(\phi,\psi)+B(\psi)}\in[0,1].
\label{eq:sc2b_beta_star}
\end{equation}
\end{corollary}

\vspace{0.05in}\noindent\textbf{Critic target augmentation is a bias--variance tradeoff.} By weighing the mismatch error, $(1-\beta_{cf})^2\,\Delta_2$, against the proxy variance, $\beta^2_{cf}\sigma^2_{\mathrm{scen}}$, the bound reveals a clear tension. Corollary ~\ref{cor:sc2b_beta_star} suggests that a "sweet spot" exists whenever both error sources are significant, mirrors the empirical trends for $\beta_{cf}$ seen in Sec.~\ref{subsec:sensitivity}.

\section{Experiments}
\label{sec:experiments}

We evaluate our method on the daily close-to-close portfolio rebalancing task to demonstrate the effectiveness and robustness of our solution, particularly in regime mixture changes.

\subsection{Experimental Setup}
\label{subsec:exp_setup}

\textbf{Datasets and universe construction.}
We use the FinRL U.S.\ equities and ETFs dataset covering 2009--2023 \citep{openfinancelab2025task1}.
Across all experiments, we enforce a chronological split: Training (2009–2017), Validation (2018–2019), and an out-of-sample Testing (2020–2023). 
A \emph{universe} is a fixed set of assets (along with their return panels) that we rebalance. To make sure our findings aren’t an artifact of any one asset set, we evaluate on 31 distinct trading universes, each containing roughly 10–20 assets. We evaluate across four categories: Market-Proxy (1 universe of broad market/sector ETFs), High-Vol (10 distinct high-volatility equity universes), Low-Vol (10 distinct low-volatility equity universes), and General (10 distinct equity universes from the broader cross-section) \citep{AngHodrickXingZhang2006CrossSectionVolatility,ZhuDingJinShen2023DissectingIVOL}.

\begin{table*}[t]
\centering
\caption{\textbf{Evaluations and comparisons of portfolio rebalancing algorithms by universe group. Values are in median [Q1, Q3].}}\vspace{-0.1in}
\label{tab:main_results}
\resizebox{\linewidth}{!}{%
\begin{tabular}{llrcccc}
\toprule
Universe & Method & Sharpe $\uparrow$ & Calmar $\uparrow$ & AnnVol $\downarrow$ & MaxDD $\downarrow$ & Turnover \\
\midrule
High-Vol & 1/N & 0.618 [0.563, 0.692] & 0.912 [0.471, 2.114] & 0.635 [0.356, 3.047] & 0.445 [0.396, 0.493] & {0.0000 [0.0000, 0.0000]} \\
 & Markowitz & 0.561 [0.476, 0.699] & 0.500 [0.317, 0.767] & 0.527 [0.258, 0.780] & 0.441 [0.411, 0.461] & 0.0256 [0.0215, 0.0283] \\
 & Inverse-Vol & 0.630 [0.576, 0.695] & 0.873 [0.411, 2.363] & 0.589 [0.316, 2.350] & 0.418 [0.404, 0.480] & 0.0175 [0.0169, 0.0182] \\
 & GMV (Ledoit-Wolf) & 0.580 [0.533, 0.645] & 0.850 [0.305, 1.397] & 0.858 [0.518, 1.076] & 0.412 [0.365, 0.449] & 0.0261 [0.0216, 0.0305] \\
 & PPO (Historical Replay) & 0.615 [0.517, 0.719] & \textbf{1.127 [0.445, 2.451]} & 0.569 [0.363, 3.674] & 0.416 [0.385, 0.498] & 0.298 [0.293, 0.299] \\
 & BootRollout--PPO & 0.598 [0.535, 0.729] & 1.112 [0.502, 2.295] & 0.647 [0.331, 3.798] & 0.414 [0.367, 0.478] & 0.2995 [0.2951, 0.2997] \\
 & SCR--PPO--Full & \textbf{0.782 [0.628, 0.948]} & 0.699 [0.519, 1.104] & \textbf{0.226 [0.199, 0.297]} & \textbf{0.238 [0.187, 0.304]} & 0.0049 [0.0044, 0.0074] \\
\addlinespace
Low-Vol & 1/N & 0.086 [0.073, 0.232] & 0.001 [-0.006, 0.075] & 0.171 [0.158, 0.180] & 0.340 [0.321, 0.368] & {0.0000 [0.0000, 0.0000]} \\
 & Markowitz & -0.084 [-0.155, -0.027] & -0.058 [-0.093, -0.039] & 0.137 [0.123, 0.147] & 0.333 [0.300, 0.344] & 0.0046 [0.0018, 0.0146] \\
 & Inverse-Vol & 0.105 [0.041, 0.155] & 0.014 [-0.015, 0.042] & 0.165 [0.143, 0.172] & 0.321 [0.276, 0.326] & 0.0143 [0.0133, 0.0152] \\
 & GMV (Ledoit-Wolf) & -0.088 [-0.131, -0.055] & -0.057 [-0.071, -0.038] & \textbf{0.112 [0.102, 0.117]} & 0.278 [0.247, 0.295] & 0.0274 [0.0215, 0.0292] \\
 & PPO (Historical Replay) & 0.079 [0.029, 0.181] & -0.004 [-0.031, 0.046] & 0.174 [0.157, 0.183] & 0.337 [0.317, 0.369] & 0.164 [0.156, 0.294] \\
 & BootRollout--PPO & 0.090 [0.054, 0.184] & 0.000 [-0.018, 0.050] & 0.176 [0.159, 0.184] & 0.343 [0.302, 0.372] & 0.2995 [0.2993, 0.2998] \\
 & SCR--PPO--Full & \textbf{0.641 [0.602, 0.707]} & \textbf{0.453 [0.400 0.569]} & 0.116 [0.109, 0.121] & \textbf{0.164 [0.150, 0.186]} & 0.0049 [0.0046, 0.0052] \\
\addlinespace
General & 1/N & 0.444 [0.376, 0.468] & 0.182 [0.148, 0.197] & 0.231 [0.229, 0.235] & 0.423 [0.414, 0.433] & {0.0000 [0.0000, 0.0000]} \\
 & Markowitz & 0.434 [0.333, 0.554] & 0.167 [0.122, 0.305] & 0.223 [0.189, 0.267] & 0.387 [0.340, 0.427] & 0.0044 [0.0018, 0.0095] \\
 & Inverse-Vol & 0.472 [0.440, 0.516] & 0.221 [0.200, 0.246] & 0.238 [0.222, 0.239] & 0.395 [0.370, 0.402] & 0.0169 [0.0159, 0.0171] \\
 & GMV (Ledoit-Wolf) & 0.200 [0.082, 0.397] & 0.063 [0.008, 0.280] & \textbf{0.148 [0.134, 0.222]} & 0.329 [0.260, 0.377] & 0.0282 [0.0263, 0.0314] \\
 & PPO (Historical Replay) & 0.435 [0.341, 0.525] & 0.188 [0.136, 0.270] & 0.238 [0.231, 0.242] & 0.409 [0.399, 0.416] & 0.194 [0.158, 0.282] \\
 & BootRollout--PPO & 0.419 [0.360, 0.517] & 0.188 [0.137, 0.239] & 0.232 [0.226, 0.235] & 0.400 [0.388 0.420] & 0.2996 [0.2995, 0.2997] \\
 & SCR--PPO--Full & \textbf{0.900 [0.768, 0.992]} & \textbf{0.972 [0.855, 1.350]} & 0.171 [0.162, 0.178] & \textbf{0.167 [0.142, 0.176]} & 0.0059 [0.0054, 0.0062] \\
\addlinespace
Market-Proxy & 1/N & 0.508 [0.508, 0.508] & 0.250 [0.250, 0.250] & 0.225 [0.225, 0.225] & 0.371 [0.371, 0.371] & {0.0000 [0.0000, 0.0000]} \\
& Markowitz & 0.381 [0.381, 0.381] & 0.165 [0.165, 0.165] & 0.210 [0.210, 0.210] & 0.362 [0.362, 0.362] & 0.0056 [0.0056, 0.0056] \\
 & Inverse-Vol & 0.494 [0.494, 0.494] & 0.241 [0.241, 0.241] & 0.217 [0.217, 0.217] & 0.361 [0.361, 0.361] & 0.0091 [0.0091, 0.0091] \\
 & GMV (Ledoit-Wolf) & 0.569 [0.569, 0.569] & 0.303 [0.303, 0.303] & 0.208 [0.208, 0.208] & 0.335 [0.335, 0.335] & 0.0468 [0.0468, 0.0468] \\
 & PPO (Historical Replay) & 0.492 [0.479, 0.523] & 0.241 [0.234, 0.268] & 0.228 [0.228, 0.228] & 0.370 [0.367, 0.372] & 0.220 [0.210, 0.267] \\
 & BootRollout--PPO & 0.496 [0.481, 0.509] & 0.245 [0.238, 0.255] & 0.227 [0.226, 0.228] & 0.365 [0.364, 0.371] & 0.2926 [0.2893, 0.2947] \\
 & SCR--PPO--Full & \textbf{1.004 [1.001, 1.014]} & \textbf{0.971 [0.928, 0.991]} & \textbf{0.175 [0.173, 0.177]} & \textbf{0.179 [0.176, 0.193]} & 0.0053 [0.0052, 0.0054] \\
\bottomrule
\end{tabular}%
}
\end{table*}

\vspace{0.05in}\noindent\textbf{Evaluation metrics.}
In line with recent related work in portfolio rebalancing and risk assessment
\citep{wang2024weightadjustment,wang2025risksensitive,choudhary2025finxplore},
we evaluate out-of-sample performance using a set of comprehensive metrics. The assessment of risk-adjusted performance employs the Sharpe ratio and the Calmar ratio, which respectively reflect the average excess return per unit of volatility and drawdown-aware performance. To assess the total risk exposure, we measure annualized volatility (AnnVol), while tail risk is analyzed via maximum drawdown (MaxDD), defined as the largest peak-to-trough loss.
Additionally, we report turnover as a measure of trading intensity and implementability, computed as the average $\ell_1$ change in portfolio weights. These metrics, taken together, show how returns and risk behave when the market environment changes and the distribution shifts.

\noindent\textbf{Baselines.} We compare against both classical portfolio baselines and RL-based solutions.
All reported metrics are computed from realized net returns on the held-out out-of-sample test window. 
\begin{itemize}[leftmargin=1.2em, nosep]
\item \emph{Classic portfolio algorithms,} including (i) equal-weight $1/N$ \citep{DeMiguel2009, GelminiUberti2024},
(ii) mean--variance \citep{Markowitz1952_MARKOWITZ, LaiYang2023Survey},
(iii) inverse-volatility weighting \citep{Chaves2011_InverseVol_RiskParityHeuristics, ShimizuShiohama2020IFV},
and (iv) global minimum-variance (GMV) with Ledoit--Wolf shrinkage \citep{Ledoit2004honey, LedoitWolf2020ANS}.
\item \emph{PPO (Historical Replay):} PPO trained only on realized market transitions, using PPO-Clip~\cite{Schulman2017PPO} with generalized advantage estimation~\cite{schulman2015gae}.
\item \emph{BootRollout--PPO (with historical simulation):} 
Inspired by historical simulation \citep{FrancqZakoian2020VHS}, we bootstrap $S$ joint return vectors from a trailing window of realized returns
(preserving cross-asset co-movements) and score portfolios with a downside-risk--penalized
objective (entropic left-tail) \citep{Fei2021ExpBellman}; PPO otherwise trains on the same realized
tape transitions with standard bootstrapping.
\item \emph{SCR--PPO--RewardOnly:} A variant of our solution in which SCR macro scenarios for reward only, with {standard} tape bootstrapping and {without} regularization.
\item \emph{SCR--PPO--NoCF:} A variant of our solution in which SCR macro scenarios for reward with regularization, but bootstrapping uses realized tape continuation only.
\item \emph{SCR–-PPO--Full:} SCR macro scenarios for reward with regularization stack and augmented critic bootstrap target ($\beta_{cf}>0$). 
\end{itemize}

\vspace{0.05in}\noindent\textbf{Shock Discovery}.
SCR conditions applied to the channel-activation vector $\boldsymbol{\chi}_t$ confirms the conditioning signal is non-degenerate and exhibits reasonable behaviour in out-of-sample. On Market-Proxy, the detector successfully re-identifies recurring channels in validation and test, producing similar $\boldsymbol{\chi}_t$ patterns. Notably, it identifies March 2020 as out-of-support when the lookback window activates channels that were not encountered during training. In instances of novelty detection, SCR resorts to acceptable nearest-neighbor resampling from the available historical data, thereby avoiding retrieval from unrelated episodes.
Table~\ref{tab:shock_ledger} presents representative entries from ShockLedger, verifying that regime summaries via $\boldsymbol{\chi}_t$ are stable and capable of identifying novel instances outside the existing support. Building on this qualitative verification, we turn to quantitative evaluation to test the effect of regime-local scenario conditioning, including generalization and value learning.

\begin{table}[t]
\centering
\caption{\textbf{Illustrative ShockLedger channels used in experiments.}}\vspace{-0.1in}
\label{tab:shock_ledger}
\resizebox{\linewidth}{!}{%
\begin{tabular}{l l l l}
\toprule
\textbf{Episode Type} & \textbf{Macro Signature} & \textbf{Top Moves} & \textbf{Activity Period} \\
\midrule
Recurring &
\begin{tabular}[c]{@{}l@{}}
\textbf{Geopolitical/Oil Shock}\\
$\mathrm{GPR}\uparrow,\ \mathrm{Brent}\uparrow,\ \mathrm{Energy}\uparrow$
\end{tabular} &
\begin{tabular}[c]{@{}l@{}}
$+$(\texttt{XLP}, \texttt{XLV});\\
$-$(\texttt{XLF}, \texttt{XLB})
\end{tabular} &
\begin{tabular}[c]{@{}l@{}}
2009--2023\\
(38 days; 14 in train)
\end{tabular}
\\
\midrule
Recurring &
\begin{tabular}[c]{@{}l@{}}
\textbf{Policy Uncertainty}\\
$\mathrm{GPR}\downarrow,\ \mathrm{EPU}\uparrow,\ \mathrm{WTI}\downarrow$
\end{tabular} &
\begin{tabular}[c]{@{}l@{}}
$+$(\texttt{XLK}, \texttt{XLV});\\
$-$(\texttt{XLE}, \texttt{XLF})
\end{tabular} &
\begin{tabular}[c]{@{}l@{}}
2009--2021\\
(18 days; 16 in train)
\end{tabular}
\\
\midrule
Novel (Anomaly) &
\begin{tabular}[c]{@{}l@{}}
\textbf{COVID-19 Liquidity Crisis}\\
\textit{No training-set channel match}\\
\textit{Flagged by novelty filter}
\end{tabular} &
\textit{Detected as novel episode} &
\begin{tabular}[c]{@{}l@{}}
March 2020\\
(Consecutive days)
\end{tabular}
\\
\bottomrule
\end{tabular}%
}
\end{table}

\subsection{Evaluation on Out-of-Sample Regimes}
\label{subsec:oos_perf}

Table~\ref{tab:main_results} presents the evaluation metrics for High-Vol, Low-Vol, General, and Market-Proxy groups. SCR--PPO--Full enhances risk-adjusted performance and tail behavior out-of-sample. It stays reliable across different regimes and cross-sectional variations.

\begin{table*}[t]
\centering
\caption{\textbf{Ablation study of our solution. We report median values across multiple runs.}} \vspace{-0.1in}
\label{tab:ablations_overall}
\footnotesize
\setlength{\tabcolsep}{3.5pt}
\renewcommand{\arraystretch}{0.95}
\resizebox{\linewidth}{!}{%
\begin{tabular}{l r ccccccc}
\toprule
Variant & Sharpe & Calmar & AnnVol & MaxDD & Turnover & Gap$_{\mathrm{final}}$ & Resid-AUC \\
\midrule
PPO (Historical Replay) & 0.457 [0.256, 0.543] & 0.222 [0.086, 0.335] & 0.231 [0.188, 0.347] & 0.381 [0.359, 0.414]  & 0.2320 [0.1600, 0.2983] & 0.526 [0.304, 1.882] & 0.217 [0.204, 0.231] \\
\addlinespace[2pt]
SCR--PPO--RewardOnly
& 0.506 [-0.054, 0.650]
& 0.364 [0.223, 0.584]
& \textbf{0.127 [0.107, 0.205]}
& 0.181 [0.148, 0.330]
& 0.0050 [0.0049, 0.0051]
& 0.0003 [0.0001, 0.0016]
& 0.102 [0.054, 0.115] \\
\addlinespace[2pt]
SCR--PPO--NoCF
& 0.647 [0.571, 0.810]
& 0.561 [0.379, 0.833]
& 0.146 [0.110, 0.217]
& 0.184 [0.158, 0.198]
& \textbf{0.0045 [0.0042, 0.0049]}
& 0.0001 [0.0000, 0.0012]
& 0.087 [0.077, 0.098] \\
SCR--PPO--Full
& \textbf{0.805 [0.635, 0.970]}
& \textbf{0.737 [0.494, 0.984]}
& 0.171 [0.122, 0.196]
& \textbf{0.179 [0.162, 0.194]}
& 0.0046 [0.0044, 0.0049]
& \textbf{0.0001 [0.0000, 0.0003]}
& \textbf{0.075 [0.061, 0.084]} \\
\bottomrule
\end{tabular}%
}
\end{table*}

\begin{figure}[t]
    \centering
    \includegraphics[width=\columnwidth]{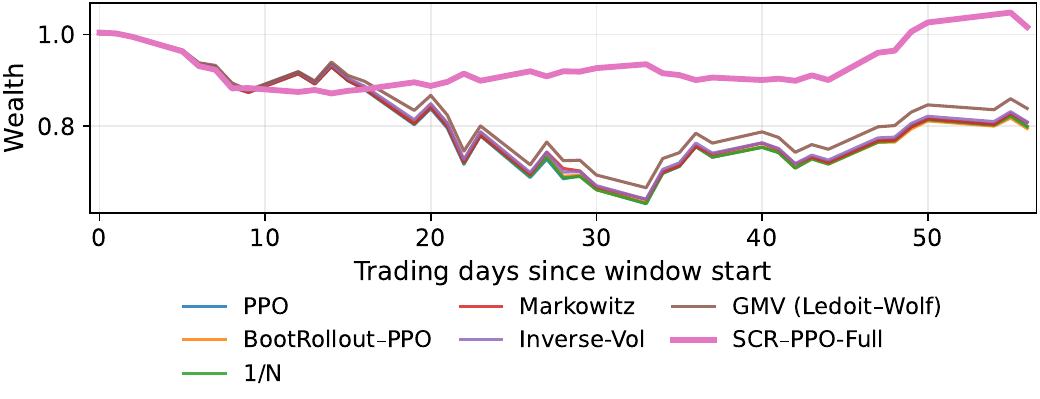}
    \includegraphics[width=\columnwidth]{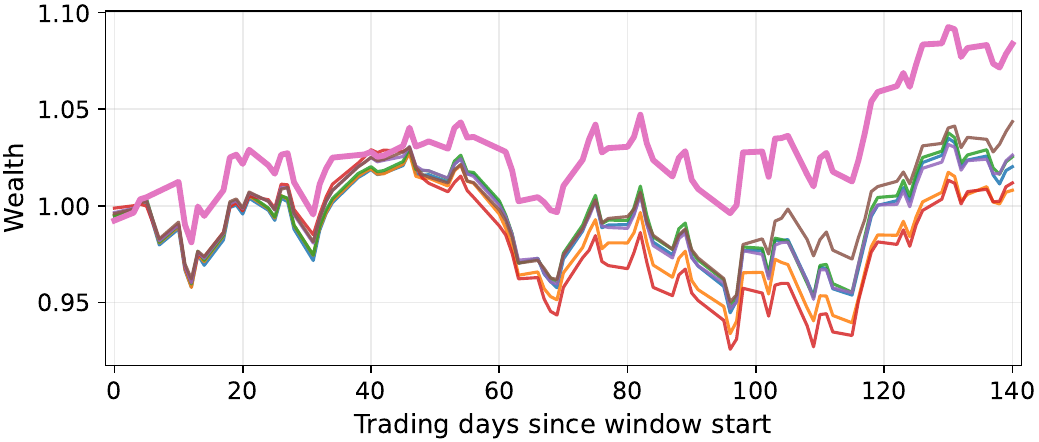}
    \caption{ We stress test under shocks. The wealths are rebased to 1 at the window start. Top: COVID-19 sell-off window. Bottom: 2021–2022 macro shock window.}

    \label{fig:robustness_windows}
\end{figure}
In comparison to PPO (Historical Replay), Sharpe shows consistent increase: $0.615 \rightarrow 0.782$ (High-Vol), $0.079 \rightarrow 0.641$ (Low-Vol), $0.435 \rightarrow 0.900$ (General), and $0.492 \rightarrow1.004$ (Market-Proxy). These enhancements are accompanied by reduced drawdowns (e.g., High-Vol MaxDD $0.416 \rightarrow 0.238$, Low-Vol $0.337 \rightarrow 0.164$ General $0.409 \rightarrow 0.167$, Market-Proxy $0.370 \rightarrow 0.179$), along with negligible turnover ($\approx 0.004$--$0.006$) in contrast to PPO’s high turnover ($\approx 0.15$--$0.30$), suggesting reduced churn without sacrificing downside control.

\vspace{0.05in}\noindent\textbf{Qualitative stress-window evidence.}
On {Market-Proxy}, Fig.~\ref{fig:robustness_windows} provides a stringent stress test: during two major crisis windows, SCR--PPO--Full protects capital and keeps drawdowns smaller.
It also stays more stable throughout the downturn.
When markets rebound, it compounds faster than both RL and classical rebalancing baselines.
Overall, it achieves higher cumulative returns under sharp regime shifts.

\subsection{Ablations and Scenario-to-Real Analysis}
\label{subsec:ablations_sim2real}

To ensure that the scenario--real mismatch and critic stability are observable, we report seed-averaged diagnostics on Market-Proxy.
Here, beyond the aggregate metrics in Sec. ~\ref{subsec:oos_perf}, we track diagnostics that correspond to the operator and mixing analysis in Sec.~\ref{sec:theory}.
We gauge scenario--real mismatch on held-out dates via scenario-to-real gap
$\mathrm{Gap}_{\mathrm{final}} \;=\; \big|\widehat J_{\mathrm{scen}}-\widehat J_{\mathrm{real}}\big|,$ 
where $\widehat J_{\mathrm{scen}}$ is the \emph{cumulative average daily return} over the test dates under SCR scenario scoring and
$\widehat J_{\mathrm{real}}$ is the corresponding cumulative average daily return realized on the tape.
We summarize critic stability by \emph{Residual-AUC}, the area under the Bellman-residual curve $\mathrm{resid}_{\ell_2}$ versus training progress.

Table~\ref{tab:ablations_overall} demonstrates the results. Relative to PPO trained on Historical Replay, introducing scenario-conditioned training via {SCR--PPO--RewardOnly} improves pooled out-of-sample robustness:
Sharpe increases ($0.457\!\rightarrow\!0.506$), drawdowns improve (MaxDD $0.381\!\rightarrow\!0.168$), and turnover drops sharply ($0.232\!\rightarrow\!0.005$).
SCR also reduces mismatch (Gap$_{\mathrm{final}}$ $0.526\!\rightarrow\!3\times 10^{-4}$) and lowers critic error (Resid-AUC $0.217\!\rightarrow\!0.102$), consistent with higher-support scenario conditioning stabilizing value learning on the logged tape.

\begin{figure}[!t]
    \centering
    \includegraphics[height=0.5\columnwidth]{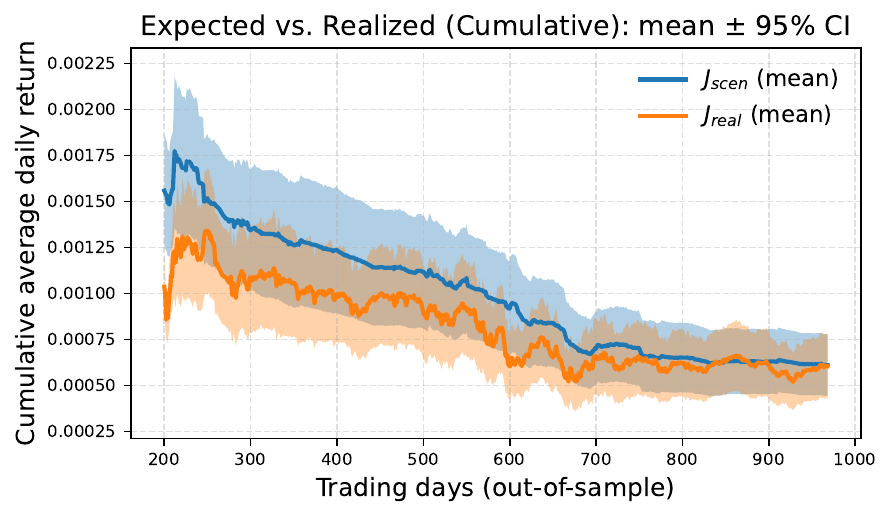}\vspace{-0.1in}
    \caption{
    \textbf{Scenario-to-real validation.}
    Cumulative average daily return under SCR scenario scoring versus realized tape returns.}

    \label{fig:sim_real_gap}
\end{figure}

\begin{figure}[!t]
    \centering
    \includegraphics[width=\columnwidth]{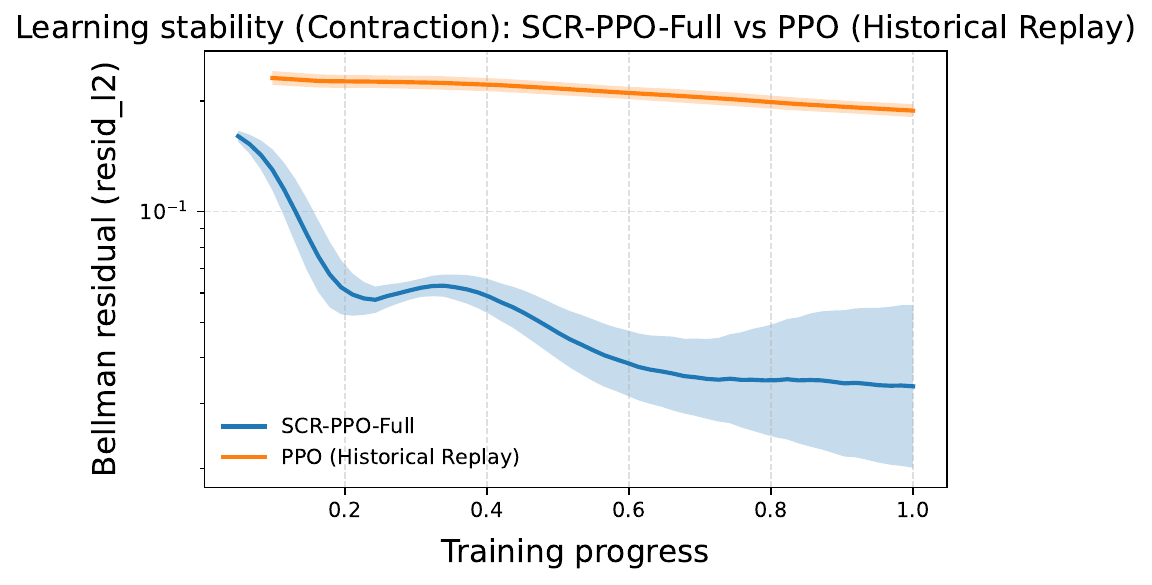}\vspace{-0.1in}
    \caption{\textbf{Critic stability under logged-tape mismatch.}
    Bellman residual $\mathrm{resid}_{\ell_2}$ versus training progress (mean $\pm$ 95\% CI). Compared to PPO (Historical Replay), SCR--PPO--Full attains smaller residuals, and improved critic stability.}

    \label{fig:critic_stability_resid_l2}
\end{figure}

Our controlled comparison {SCR--PPO--NoCF $\rightarrow$ SCR--PPO--Full} holds the scenario mechanism and regularization stack fixed and changes only the critic bootstrap target, isolating continuation mismatch (Lemma~\ref{lem:reward_cancels}).
Enabling counterfactual continuation mixing improves pooled risk-adjusted performance (Sharpe $0.647\!\rightarrow\!0.805$, Calmar $0.561\!\rightarrow\!0.737$) and reduces critic error (Resid-AUC $0.087\!\rightarrow\!0.075$).
On Market-Proxy, we report mean $\pm$ 95\% confidence intervals for $\widehat J_{\mathrm{scen}}$ and $\widehat J_{\mathrm{real}}$ (Fig.~\ref{fig:sim_real_gap}),
and evaluate critic stability via TD-residual contraction (Fig.~\ref{fig:critic_stability_resid_l2}), as motivated by Theorem~\ref{thm:sc2b_mse}.
These seed-averaged plots make the mismatch and critic contraction predicted by our theory directly observable (Figs.~\ref{fig:sim_real_gap}--\ref{fig:critic_stability_resid_l2}).

\begin{table}
\caption{\textbf{Sensitivity to $\beta_{cf}$. We report median Sharpe here.}} \vspace{-0.1in}
\label{tab:beta_sweep}
\scriptsize
\setlength{\tabcolsep}{4pt}
\renewcommand{\arraystretch}{0.9}
\begin{tabular*}{\columnwidth}{@{\extracolsep{\fill}}lcccc}
\toprule
Setting & High-Vol & Low-Vol & General & Market-Proxy \\
\midrule
$\beta_{cf}=0.0$ & 0.736 & 0.598 & 0.640 & 0.821 \\
$\beta_{cf}=0.25$ & 0.529 & 0.529 & 0.716 & 0.823 \\
$\beta_{cf}=0.5$  & \textbf{0.782} & \textbf{0.641} & \textbf{0.900} & \textbf{1.004} \\
$\beta_{cf}=0.75$ & 0.518 & 0.509 & 0.671 & 0.838 \\
$\beta_{cf}=1.0$  & 0.563 & 0.584 & 0.645 & 0.842 \\
\bottomrule
\end{tabular*}
\end{table}

\subsection{Sensitivity Analysis}
\label{subsec:sensitivity}
\noindent\textbf{Counterfactual mixing $\beta_{cf}$.}
We show in Table~\ref{tab:beta_sweep} that using a moderate coefficient results in consistent performance across different universes. In contrast, setting the coefficient in both extremes, i.e., $\beta_{cf}=0$ uses only continuation mismatch, while $\beta_{cf}=1$ relies entirely on counterfactual continuation and can amplify approximation error, resulting in degraded performance. This observation aligns with the bias--variance trade-off in Theorem~\ref{thm:sc2b_mse}.

\section{Conclusion}
\label{sec:conclusion}

To improve the return and robustness of reinforcement learning-based portfolio rebalancing, we propose Scenario-Context Rollout (SCR), which is a macroeconomic-guided feedback mechanism to produce scenario returns. To effectively train the RL agent using historical tapes, we analyzed how combining scenario-based rewards with tape-realized transitions induces a continuation mismatch. To address this issue, we augment the critic target using counterfactual continuation; doing so provides a principled bias--variance trade-off and stabilizes critic updates. Our experiments show that the proposed solution improves out-of-sample Sharpe ratios and reduces drawdowns compared with strong baselines. Meanwhile, our solution remains effective across market regimes.

\appendix

\bibliographystyle{unsrtnat}
\bibliography{ref}
\section{Additional Proofs}
\label{app:proofs}

\begin{definition}[Induced continuation distributions]
\label{def:cont_kernels}
Fix a decision-time state $\phi$ and action $a$ with executed control $\mathbf w=\Pi(a)$.
Define the realized continuation distribution as
\begin{equation}
\begin{aligned}
K_{\mathrm{real}}(B\mid \phi,a)
&:= \int \mathbf 1\{E_{B}(\phi,\mathbf w,\mathbf r)\}\,
P_{\mathrm{real}}(d\mathbf r\mid \phi), \\
E_{B}(\phi,\mathbf w,\mathbf r)
&:= \bigl(\mathrm{Upd}(\phi,\mathbf w,\mathbf r)\in B\bigr).
\end{aligned}
\end{equation}

and the scenario-consistent continuation distribution by
\begin{equation}
\begin{aligned}
K_{\mathrm{scen}}(B\mid \phi, \psi, a)
&:= \int \mathbf 1\{E_{B}(\phi,\mathbf w,\mathbf R)\}\,
P_{\mathrm{scen}}(d\mathbf R\mid \psi), \\
E_{B}(\phi,\mathbf w,\mathbf R)
&:= \bigl(\mathrm{Upd}(\phi,\mathbf w,\mathbf R)\in B\bigr).
\end{aligned}
\end{equation}

for any measurable set $B$ in the state space.
\end{definition}

\subsection{Proof of Lemma~\ref{lem:reward_cancels}}
\label{app:proof_reward_cancels}

\begin{proof}
Fix $\phi$ and a policy $\pi$. By definition,
\[
(T_{\mathrm{hyb}}^\pi V)(\phi)
=
\mathbb E\!\left[
r(\phi,\mathbf w;\mathbf R)+\delta\,V(\mathrm{Upd}(\phi,\mathbf w,\mathbf r))
\ \middle|\ \phi
\right],
\]
where $a\sim\pi(\cdot\mid\phi)$, $\mathbf w=\Pi(a)$, $\mathbf R\sim P_{\mathrm{scen}}(\cdot\mid\psi)$ and
$\mathbf r\sim P_{\mathrm{real}}(\cdot\mid\phi)$.
Similarly,
\[
(T_{\mathrm{scen}}^\pi V)(\phi)
=
\mathbb E\!\left[
r(\phi,\mathbf w;\mathbf R)+\delta\,V(\mathrm{Upd}(\phi,\mathbf w,\mathbf R))
\ \middle|\ \phi
\right].
\]
Subtracting, the reward terms cancel since both operators evaluate $r(\phi,\mathbf w;\cdot)$ on the same scenario draw $\mathbf R$:
\[
\begin{gathered}
\Delta^\pi V(\phi) := (T_{\mathrm{hyb}}^\pi V - T_{\mathrm{scen}}^\pi V)(\phi),\\
\phi^+(\phi,a,x) := \mathrm{Upd}(\phi,\Pi(a),x).
\end{gathered}
\]
\begin{equation}
\begin{aligned}
\Delta^\pi V(\phi)
&=
\delta\,\mathbb E_{a\sim\pi(\cdot\mid\phi)}\!\left[\Delta_V(\phi,\psi,a)\right], \\
\Delta_V(\phi,\psi,a)
&:=
\mathbb E_{\mathbf r}\,V\!\big(\phi^+(\phi,a,\mathbf r)\big)
-
\mathbb E_{\mathbf R}\,V\!\big(\phi^+(\phi,a,\mathbf R)\big).
\end{aligned}
\end{equation}

Using the induced continuation distributions $K_{\mathrm{real}}(\cdot\mid\phi,a)$ and $K_{\mathrm{scen}}(\cdot\mid\phi,\psi,a)$ from
Definition~\ref{def:cont_kernels}, rewrite
\begin{equation}
\mathbb E_{\mathbf r}\,V\!\big(\mathrm{Upd}(\phi,\Pi(a),\mathbf r)\big)
= \int V(\phi')\,K_{\mathrm{real}}(d\phi'\mid\phi,a).
\end{equation}

\begin{flalign}
&\mathbb E_{\mathbf R}\,V\!\big(\mathrm{Upd}(\phi,\Pi(a),\mathbf R)\big) && \\
&= \int V(\phi')\,K_{\mathrm{scen}}\!\left(d\phi'\mid\phi,\psi,a\right). &&
\end{flalign}

Therefore
\[
\begin{aligned}
\pi_\phi 
&:= \pi(\cdot\mid\phi),\\
\Delta K_{\phi,\psi,a}(d\phi')
&:= K_{\mathrm{real}}(d\phi'\mid\phi,a)-K_{\mathrm{scen}}(d\phi'\mid\phi,\psi,a).\\
\langle V,\Delta K_{\phi,\psi,a}\rangle
&:= \int V(\phi')\,\Delta K_{\phi,\psi,a}(d\phi').
\end{aligned}
\]

where $\Delta K:=K_{\mathrm{real}}-K_{\mathrm{scen}}$.
If $P_{\mathrm{real}}(\cdot\mid\phi)=P_{\mathrm{scen}}(\cdot\mid\psi)$, then $K_{\mathrm{real}}=K_{\mathrm{scen}}$ and the difference is zero.
If $\mathrm{Upd}(\phi,\mathbf w,\cdot)$ is outcome-independent, then both induced kernels coincide as well.
\end{proof}

\subsection{Proof of Proposition~\ref{prop:op_gap}}
\label{app:proof_op_gap}

\begin{proof}
Fix $\phi$. Fix an $o$ such that $\Phi(o)=\phi$, and set $\psi:=\Psi(o)$. By Lemma~\ref{lem:reward_cancels},
\[
\begin{aligned}
m_{\mathrm{real}}(\phi,a)
&:= \mathbb E_{\mathbf r\sim P_{\mathrm{real}}(\cdot\mid\phi)} f_{\phi,a}(\mathbf r),\\
m_{\mathrm{scen}}(\phi,\psi,a)
&:= \mathbb E_{\mathbf R\sim P_{\mathrm{scen}}(\cdot\mid\psi)} f_{\phi,a}(\mathbf R).
\end{aligned}
\]
Let
\[
\Delta^\pi V(\phi) := (T_{\mathrm{hyb}}^\pi V - T_{\mathrm{scen}}^\pi V)(\phi).
\]
Then
\[
\Delta^\pi V(\phi)
= \delta\,\mathbb E_{a\sim\pi(\cdot\mid\phi)}\!\Big[
m_{\mathrm{real}}(\phi,a)-m_{\mathrm{scen}}(\phi,\psi,a)
\Big]
\]
where $f_{\phi,a}(\mathbf x):=V(\mathrm{Upd}(\phi,\Pi(a),\mathbf x))$.

By assumption, $\mathbf x\mapsto \mathrm{Upd}(\phi,\Pi(a),\mathbf x)$ is $L_h$-Lipschitz through the outcome-updated memory component and
$V$ is $L_V$-Lipschitz in that component, hence $f_{\phi,a}$ is $(L_VL_h)$-Lipschitz under $\|\cdot\|_2$.
By Kantorovich--Rubinstein duality for $W_1$ (with cost $d(\mathbf x,\mathbf y)=\|\mathbf x-\mathbf y\|_2$) \citep{kantorovich1958space},
\[
P_{\mathrm{real}}^\phi := P_{\mathrm{real}}(\cdot\mid\phi),
\qquad
P_{\mathrm{scen}}^\psi := P_{\mathrm{scen}}(\cdot\mid\psi).
\]
 
\[
\begin{aligned}
\Bigl|
\mathbb E_{\mathbf r} f_{\phi,a}(\mathbf r)
-
\mathbb E_{\mathbf R} f_{\phi,a}(\mathbf R)
\Bigr|
&\le (L_VL_h)\, W_1\!\bigl(P_{\mathrm{real}}^\phi,\,P_{\mathrm{scen}}^\psi\bigr).
\end{aligned}
\]

Taking expectation over $a\sim\pi(\cdot\mid\phi)$, multiplying by $\delta$, and then taking the supremum over $o$ yields
\[
\Delta_W
:= \sup_{o}\ 
W_1\!\Big(P_{\mathrm{real}}(\cdot \mid \Phi(o)),\;
          P_{\mathrm{scen}}(\cdot \mid \Psi(o))\Big).
\]
\[
\|T_{\mathrm{hyb}}^\pi V - T_{\mathrm{scen}}^\pi V\|_\infty
\le \delta L_VL_h\,\Delta_W.
\]

\end{proof}

\subsection{Proof of Corollary~\ref{cor:fixed_point_bias}}
\label{app:proof_fixed_point_bias}

\begin{proof}
Let $V_{\mathrm{hyb}}^\pi$ and $V_{\mathrm{scen}}^\pi$ be the fixed points of $T_{\mathrm{hyb}}^\pi$ and $T_{\mathrm{scen}}^\pi$.
Since both operators are $\delta$-contractions in $\|\cdot\|_\infty$,
\[
\begin{aligned}
\|V_{\mathrm{hyb}}^\pi - V_{\mathrm{scen}}^\pi\|_\infty
&= \bigl\|T_{\mathrm{hyb}}^\pi V_{\mathrm{hyb}}^\pi
      - T_{\mathrm{scen}}^\pi V_{\mathrm{scen}}^\pi\bigr\|_\infty \\
&\le \bigl\|T_{\mathrm{hyb}}^\pi V_{\mathrm{hyb}}^\pi
      - T_{\mathrm{hyb}}^\pi V_{\mathrm{scen}}^\pi\bigr\|_\infty \\
&\quad + \bigl\|(T_{\mathrm{hyb}}^\pi - T_{\mathrm{scen}}^\pi)
      V_{\mathrm{scen}}^\pi\bigr\|_\infty .
\end{aligned}
\]

The first term is at most $\delta\|V_{\mathrm{hyb}}^\pi - V_{\mathrm{scen}}^\pi\|_\infty$ by contraction.
Rearranging gives
\[
\|V_{\mathrm{hyb}}^\pi - V_{\mathrm{scen}}^\pi\|_\infty
\le \frac{1}{1-\delta}\|(T_{\mathrm{hyb}}^\pi - T_{\mathrm{scen}}^\pi)V_{\mathrm{scen}}^\pi\|_\infty.
\]
Apply Proposition~\ref{prop:op_gap} with $V=V_{\mathrm{scen}}^\pi$ to obtain \eqref{eq:fixed_point_bias}.
\end{proof}

\subsection{Proof of Theorem~\ref{thm:sc2b_mse}}
\label{app:proof_sc2b_mse}
\begin{proof}
Fix $(\phi,\psi,a)$ and write $\mathbf w=\Pi(a)$.
Define
\[
f(\mathbf x) := V_{\mathrm{scen}}^\pi\!\big(\mathrm{Upd}(\phi,\mathbf w,\mathbf x)\big).
\]
By Assumption~\ref{ass:sc2b_reg}, $\mathbf x\mapsto \mathrm{Upd}(\phi,\mathbf w,\mathbf x)$ is $L_h$-Lipschitz through the outcome-updated memory component
and $V_{\mathrm{scen}}^\pi$ is $L_V$-Lipschitz in that component, hence $f$ is $(L_VL_h)$-Lipschitz under $\|\cdot\|_2$.

Let $\mathbf R\sim P_{\mathrm{scen}}(\cdot\mid\psi)$, $\mathbf r\sim P_{\mathrm{real}}(\cdot\mid\phi)$, and $\mu_\psi=\mathbb E[\mathbf R\mid\psi]$.
Since both $Y^\star$ and $Y_{\beta_{\mathrm{cf}}}$ use the same reward term $r(\phi,\mathbf w;\mathbf R)$, the deviation enters only via continuation:
\[
\begin{aligned}
Y_r - Y^\star &= \delta\big(f(\mathbf r)-f(\mathbf R)\big),\\
Y_c - Y^\star &= \delta\big(f(\mu_\psi)-f(\mathbf R)\big).
\end{aligned}
\]
Let $\beta=\beta_{cf}$ and $\alpha=1-\beta$. Then
\[
Y_{\beta}-Y^\star=\alpha(Y_r-Y^\star)+\beta(Y_c-Y^\star).
\]
Using $(u+v)^2\le 2u^2+2v^2$ yields
\[
(Y_{\beta}-Y^\star)^2
\le
2\alpha^2(Y_r-Y^\star)^2 + 2\beta^2(Y_c-Y^\star)^2.
\]
Taking conditional expectations,
\begin{equation}
\label{eq:pf_split}
\begin{aligned}
\mathbb E\!\left[(Y_{\beta}-Y^\star)^2 \mid \phi,\psi,a\right]
\;&\le\;
2\alpha^2\,\mathbb E\!\left[(Y_r-Y^\star)^2 \mid \phi,\psi,a\right] \\
&\quad+\;
2\beta^2\,\mathbb E\!\left[(Y_c-Y^\star)^2 \mid \phi,\psi,a\right].
\end{aligned}
\end{equation}

For the realized-continuation term, let $\gamma^\star$ be an optimal coupling of
$P_{\mathrm{real}}(\cdot\mid\phi)$ and $P_{\mathrm{scen}}(\cdot\mid\psi)$, so that
\[
\Delta_2(\phi,\psi)
=
W_2^2\!\Big(P_{\mathrm{real}}(\cdot\mid\phi),P_{\mathrm{scen}}(\cdot\mid\psi)\Big)
=
\mathbb E_{\gamma^\star}\big[\|\mathbf r-\mathbf R\|_2^2\big].
\]
By Lipschitzness of $f$,
\[
|f(\mathbf r)-f(\mathbf R)| \le (L_VL_h)\|\mathbf r-\mathbf R\|_2,
\]
hence
\[
\begin{aligned}
\mathbb E\!\left[(Y_r-Y^\star)^2 \mid \phi,\psi,a\right]
&=
\delta^2\,\mathbb E\!\left[(f(\mathbf r)-f(\mathbf R))^2 \mid \phi,\psi,a\right] \\
&\le
\delta^2(L_VL_h)^2\,\Delta_2(\phi,\psi).
\end{aligned}
\]

For the proxy-continuation term, by Lipschitzness,
\[
|f(\mu_\psi)-f(\mathbf R)| \le (L_VL_h)\|\mu_\psi-\mathbf R\|_2,
\]
so
\[
\begin{aligned}
\mathbb E\!\left[(Y_c-Y^\star)^2 \mid \phi,\psi,a\right]
&=
\delta^2\,\mathbb E\!\left[(f(\mu_\psi)-f(\mathbf R))^2 \mid \phi,\psi,a\right] \\
&\le
\delta^2(L_VL_h)^2\,\mathbb E\!\left[\|\mathbf R-\mu_\psi\|_2^2 \mid \psi\right] \\
&=
\delta^2(L_VL_h)^2\,\sigma^2_{\mathrm{scen}}(\psi).
\end{aligned}
\]

Substitute into \eqref{eq:pf_split} to obtain \eqref{eq:sc2b_mse_bound}.
\end{proof}

\subsection{Proof of Corollary~\ref{cor:sc2b_beta_star}}
\label{app:proof_beta_star}

\begin{proof}
We derive the optimal mixing weight by minimizing 
\[
g(\beta) := (1-\beta)^2A(\phi,\psi)+\beta^2B(\psi),\qquad \beta\in[0,1].
\]
Upon expansion, we obtain
$g(\beta)=A(\phi,\psi)-2A(\phi,\psi)\beta+(A(\phi,\psi)+B(\psi))\beta^2$.
Next, we differentiate this expression and set the derivative equal to zero:
\[
\begin{aligned}
g'(\beta)
&= -2A(\phi,\psi) + 2\bigl(A(\phi,\psi)+B(\psi)\bigr)\beta = 0,\\
\Rightarrow\qquad
\beta^\star
&= \frac{A(\phi,\psi)}{A(\phi,\psi)+B(\psi)}.
\end{aligned}
\]
Given that both $A(\phi,\psi)$ and $B(\psi)$ are non-negative, it follows that  $\beta^\star\in[0,1]$.
\end{proof}

\end{document}